%% file: main.tex
\documentclass[10pt,twocolumn,letterpaper]{article}

\usepackage{cvpr}              

\usepackage{tikz}
\usepackage{comment}
\usepackage{cite}
\usepackage{amsmath,amssymb} 
\usepackage{color}
\usepackage{kotex}
\usepackage{varwidth}
\usepackage{siunitx} 
\usepackage{makecell}
\usepackage{verbatim}
\usepackage{multirow}
\usepackage{booktabs}
\usepackage{adjustbox}
\usepackage{array}
\usepackage{bbm}
\usepackage{mathtools}
\usepackage{grffile}
\usepackage{enumitem}
\usepackage{kotex}
\usepackage{accents}
\usepackage[accsupp]{axessibility}
\usepackage[dvipsnames]{xcolor}

\definecolor{cvprblue}{rgb}{0.21,0.49,0.74}
\usepackage{amsthm}
\usepackage[pagebackref,breaklinks,colorlinks,allcolors=cvprblue]{hyperref}
\newcommand{\post}[1]{\textcolor{ForestGreen}{\textbf{$+ #1$}}}
\newcommand{\negt}[1]{\textcolor{BrickRed}{\textbf{$- #1$}}}

\newtheorem{proposition}{Proposition}

\def\paperID{44237} 
\def\confName{CVPR}
\def\confYear{2026}

\title{OPRO: Orthogonal Panel-Relative Operators \\ for Panel-Aware In-Context Image Generation
}
\author{
Sanghyeon Lee, \quad Minwoo Lee, \quad Euijin Shin, \quad Kangyeol Kim, \quad Seunghwan Choi, \quad Jaegul Choo\\
Korea Advanced Institute of Science and Technology (KAIST)\\
Daejeon, Korea\\
{\tt\small  \{shlee6825, minwoo011015, ejshin0310, kangyeolk, shadow2496, jchoo\}@kaist.ac.kr}
}

\begin{document}
\maketitle
\input{sec/0_abstract}    
\input{sec/1_intro}

\input{sec/2_references}
\input{sec/3_methods}
\input{sec/4_experiments}
\input{sec/5_conclusions}

{
    \small
    \bibliographystyle{ieeenat_fullname}
    \bibliography{main}
}

\input{sec/X_suppl_arxiv}

\end{document}

%% file: sec/0_abstract.tex
\begin{abstract}
We introduce a parameter-efficient adaptation method for panel-aware in-context image generation with pre-trained diffusion transformers. The key idea is to compose learnable, panel-specific orthogonal operators onto the backbone's frozen positional encodings. This design provides two desirable properties: (1) isometry, which preserves the geometry of internal features, and (2) same-panel invariance, which maintains the model’s pre-trained intra-panel synthesis behavior. Through controlled experiments, we demonstrate that the effectiveness of our adaptation method is not tied to a specific positional encoding design but generalizes across diverse positional encoding regimes. By enabling effective panel-relative conditioning, the proposed method consistently improves in-context image-based instructional editing pipelines, including state-of-the-art approaches.
\end{abstract}

%% file: sec/1_intro.tex
\section{Introduction}
\label{sec:intro}

Large-scale pre-trained models, from diffusion models~\citep{ho2020denoising, rombach2022high} to more recent diffusion transformers (DiTs)~\citep{peebles2023scalable, chen2023pixart, esser2024scaling}, have achieved state-of-the-art results in high-fidelity image generation. A promising application for these models is in-context image generation (ICG). In this paradigm, the model adapts its outputs to visual examples provided in context, much as large language models (LLMs) use textual prompts for in-context learning~\citep{brown2020language, wei2022emergent}. Enabling such flexible, example-based control has long been a research topic in image generation, often referred to as exemplar- or reference-based generation~\citep{Gatys_2016_CVPR, park2019SPADE}.

\begin{figure*}[t]
  \centering
  \includegraphics[width=\linewidth, trim=0in 4.5in 0in 4.5in, clip]{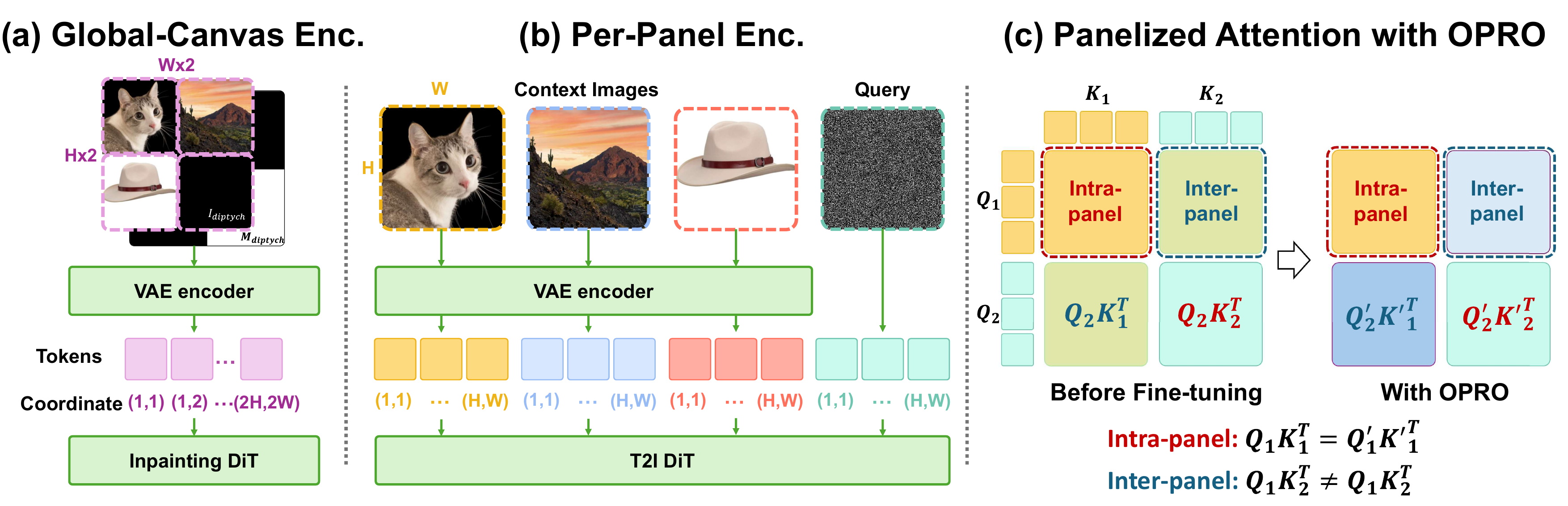}  
\caption{\textbf{Two positional regimes in tiled ICG and the role of OPRO in panelized attention.}
(a) \textbf{Global-canvas encoding}: Inpainting-based DiTs treat the tiled layout as a single image on a single global coordinate grid, so different panels become disjoint regions of a unified canvas.
(b) \textbf{Per-panel encoding}: T2I-based methods encode each panel in its own local frame and then fuse context features into target generation through attention, reusing the same coordinate range across panels. 
(c) In panelized attention, diagonal blocks are intra-panel and off-diagonal blocks are inter-panel. OPRO preserves the intra-panel blocks while modulating the inter-panel blocks.}
  \label{fig:main_intro_icg}
\end{figure*}

A typical setup for ICG is the tiled-panel layout~\citep{Diptych, ICL-Lora}, which arranges one or more context panels and a target query panel as separate visual regions. As illustrated in \cref{fig:main_intro_icg}, existing ICG approaches are commonly based on two main paradigms: leveraging inpainting-based DiTs~\citep{fluxfill} and text-to-image (T2I) DiTs~\citep{flux}.

ICG methods based on inpainting DiTs~\citep{Diptych, song2025insert, mao2025ace++, zhang2025context} formulate the generation process as a spatial completion problem. These methods place the context panels on a unified canvas, mask the target panel, and generate it conditioned on the visible panels and the text instruction. Because inpainting-based ICG methods process the entire tiled layout as a single image, positional encodings~
\citep{su2024roformer} are assigned to a single global coordinate grid shared across all panels as illustrated in~\cref{fig:main_intro_icg}(a).

Alternatively, ICG frameworks based on T2I DiTs~\citep{ICL-Lora, huang2024group, uno} incorporate context images through feature injection or inversion techniques~\citep{ddiminversion}. Rather than arranging all panels on a unified canvas, T2I-based methods first encode or invert the context images separately and then fuse the resulting latent representations into the target-generation process through attention~\citep{vaswani2017attention}. To condition the target query, these methods concatenate the extracted context representations within the attention pathway of the target-generation stream. Each panel uses its own coordinate system before inter-panel fusion as illustrated in~\cref{fig:main_intro_icg}(b). This allows both the reference and target panels to share the same scale for positioning.

Despite the distinct structural choices of global-canvas and per-panel encodings, both remain panel-agnostic at the attention level. In inpainting-based ICG, the tiled layout is processed on a single global coordinate grid; tokens from different panels are treated as distant regions within a single canvas rather than members of distinct panels. In T2I-based ICG, context features are encoded in separate streams and fused into target generation through attention; because each stream is position-encoded in its own local frame, tokens from different panels can share identical positional indices. In neither case does the attention mechanism receive an explicit signal indicating whether a token pair is intra-panel or inter-panel. Consequently, the adapter must simultaneously learn inter-panel relation transfer and preserve the backbone's pre-trained intra-panel synthesis behavior, creating a dual burden for adaptation.

This observation motivates an adaptation mechanism that explicitly distinguishes inter-panel from intra-panel interactions in attention. To this end, we introduce the Orthogonal Panel-Relative Operator (OPRO), a parameter-efficient fine-tuning (PEFT) method that applies learnable, panel-specific orthogonal operators to the backbone's position-aware queries and keys. OPRO is designed to preserve pre-trained intra-panel behavior while introducing a learnable panel-relative modulation for inter-panel interactions as illustrated in~\cref{fig:main_intro_icg}(c). This design is guided by two key properties. First, \emph{isometry} (\cref{prop:isometry}) preserves the norms of the transformed queries and keys, preventing unintended rescaling of attention logits and thereby maintaining the backbone's feature geometry during fine-tuning. Second, \emph{invariance on} the same panel (\cref{prop:same-panel-invariance}) guarantees that the attention scores between tokens of the same panel remain identical to those of the pre-trained backbone. Together, these properties allow the adapter to focus its capacity on inter-panel transfer without perturbing pre-trained intra-panel synthesis.

Our main contributions are summarized as follows:
\begin{itemize}
\item We propose OPRO, a parameter-efficient panel-relative adaptation method for tiled in-context generation, which applies learnable, panel-specific orthogonal operators to the backbone's position-aware queries and keys.

\item We establish that OPRO provides two exact guarantees for structured adaptation: \emph{same-panel invariance}, which preserves pre-trained intra-panel attention, and \emph{isometry}, which preserves feature norms and avoids unintended rescaling of attention logits.

\item We introduce a two-stage compositional reasoning benchmark to perform a controlled analysis of OPRO and evaluate its consistent performance improvements across diverse positional regimes, including APE, RoPE~\citep{su2024roformer}, LieRE~\citep{ostmeierliere}, and ComRoPE~\citep{yu2025comrope}.

\item We demonstrate that OPRO is effective in real-world ICG by improving instructional image editing, including gains over state-of-the-art methods~\citep{zhang2025context, uno} on MagicBrush~\citep{zhang2023magicbrush}.
\end{itemize}

%% file: sec/2_references.tex
 \section{Related Work}
\label{sec:related}

\subsection{In-Context Image Generation}
Motivated by the in-context learning phenomenon in large language models~\citep{brown2020language, wei2022emergent}, early ICG methods sought to replicate similar behavior in diffusion models. Prompt Diffusion~\citep{wang2023context}, iPromptDiff~\citep{chen2023improving}, and Context Diffusion~\citep{najdenkoska2024context} are trained on curated visual or textual queries/instructions paired with targets. These methods integrate the encoded features into the diffusion backbone and optimize the model under a standard diffusion objective~\citep{ho2020denoising}. Although effective within their specific training settings, these approaches are better characterized as training-time context conditioning. Their ability to generalize to unseen tasks remains to be thoroughly established.

Recent works, such as IC-LoRA~\citep{ICL-Lora} and Diptych~\citep{Diptych}, highlight an emerging tiled-canvas in-context regime for large-scale diffusion transformers and inpainting models~\citep{peebles2023scalable, chen2023pixart, fluxfill}. In these setups, reference images and the textual query are arranged on a single canvas, and the model empirically transfers concepts, styles, and identities without full model retraining. IC-LoRA exploits this behavior via PEFT on tiled canvases, providing an alternative to explicit conditioning methods such as IP-Adapter~\citep{ye2023ip}. Diptych builds on high-capacity inpainting models (e.g., FluxFill~\citep{fluxfill}) with a left-to-right canvas layout, where background removal reduces leakage and attention amplification improves the fidelity of the conditioned region. 

In parallel, PEFT has emerged as a practical approach to adapt robust text-to-image backbones while retaining their zero-shot behavior. ACE++~\citep{mao2025ace++} proposes an instruction-based diffusion framework that extends long-context inpainting-style inputs to diverse generation and editing tasks, offering both full and lightweight fine-tuning variants. InsertAnything~\citep{song2025insert} introduces a unified DiT-based reference insertion framework whose in-context editing mechanism treats reference images as contextual conditioning via multimodal attention. ICEdit~\citep{zhang2025context} attains strong instruction-guided editing performance with PEFT and Mixture-of-Experts~\citep{jacobs1991adaptive}.  
UNO~\citep{uno} proposes a progressive synthesis pipeline that harnesses the intrinsic in-context capabilities of diffusion transformers. 

\subsection{Orthogonal Relative Positional Encodings}
Rotary position embedding (RoPE)~\citep{su2024roformer} encodes a relative offset as a block-diagonal rotation, making the attention kernel equivariant to translations along the modeled axes. RoPE has been widely adopted in language models~\citep{touvron2023llama, chowdhery2023palm, jiang2024mixtral} for long-context stability and extensions.

This core idea was naturally adapted to 2D token layouts for vision. While early applications used simple axis-wise frequencies, later work proposed mixed axial frequencies to capture diagonals and perspective effects~\citep{heo2024rotary}. 2D relative positional encodings are now a foundational component in modern large-scale Diffusion Transformers, such as PixArt-$\alpha$~\citep{chen2023pixart} and the MMDiT architecture used in Flux~\citep{esser2024scaling}. Beyond these 2D adaptations, the concept of orthogonal rotation has been generalized further: Lie group-based approaches (LieRE)~\citep{ostmeierliere} model rotations directly in higher-dimensional subspaces, and commuting-angle formulations (ComRoPE)~\citep{yu2025comrope} learn trainable rotation spectra with guaranteed commutativity. 
While these foundational encodings effectively model continuous pixel displacements through rotational transformations, they remain inherently agnostic to discrete canvas partitions. We propose to augment this mathematical structure by injecting an explicit panel-aware phase into the position-aware representations. 

%% file: sec/3_methods.tex
\section{Method}
\label{sec:method}
We address in-context image generation by modifying position-aware attention to allow panel identity to influence cross-panel interactions while preserving the backbone’s pre-trained same-panel behavior. Our key idea is to apply a learnable, panel-specific orthogonal modulation to the backbone’s frozen, position-aware queries and keys.
Our approach is detailed as follows. We first formalize tiled-panel ICG with position-aware attention in~\cref{sec:method_setup}. We then introduce the Orthogonal Panel-Relative Operator (OPRO) and its core properties in~\cref{sec:method_opro}. Finally,~\cref{sec:method_param} describes our parameterization and zero-interference initialization strategy.

\subsection{Tiled-Panel ICG with Position-Aware Attention}
\label{sec:method_setup}

We consider a tiled layout partitioned into $P$ panels. After patchification and flattening, the model processes a token sequence of length $L$. Each token $i \in \{1,\dots,L\}$ is associated with two attributes: a panel index $p(i) \in \{1,\dots,P\}$ and a spatial coordinate $x_i$ in its native positional frame.

Let $\tilde{ q}_i, \tilde{ k}_i \in \mathbb{R}^{d_h}$ denote the backbone's frozen, \emph{position-aware} query and key representations for token $i$ in a given attention head of dimension $d_h$. These representations are obtained after the backbone applies its native positional mechanism. Depending on the backbone family, this positional mechanism may arise from a single global coordinate system over a tiled canvas or from per-panel positional encoding followed by cross-panel fusion. We keep this definition general and do not assume a specific positional encoding form.

The standard attention score between tokens $i$ and $j$ is
\begin{equation}
s_{ij} = \frac{\langle \tilde{ q}_i, \tilde{ k}_j \rangle}{\sqrt{d_h}}.
\label{eq:base_attn}
\end{equation}

In tiled-panel ICG, standard PEFT modules update shared attention projections but do not explicitly encode whether a token pair is same-panel or cross-panel in the attention logit. Consequently, panel-agnostic adaptation methods (e.g., LoRA) face a dual burden: they must acquire inter-panel retrieval capabilities while simultaneously preserving the pre-trained intra-panel geometry of the backbone.

\begin{figure*}[h!]
  \centering
  \includegraphics[width=1\linewidth, trim=0in 5.1in 0in 5.1in, clip]{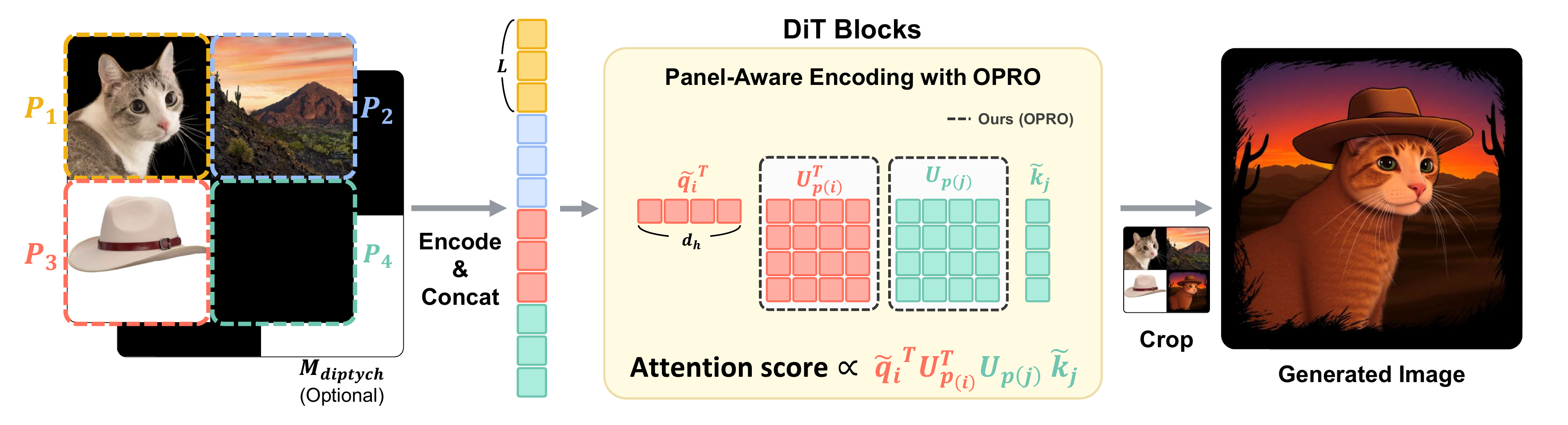}  
\caption{\textbf{Overview of OPRO for tiled-panel in-context image generation.}
The proposed framework partitions a tiled canvas into $P$ panels and processes them as a single token sequence. Within each attention layer of a backbone, OPRO modulates the position-aware queries ($\tilde{q}_i$) and keys ($\tilde{k}_j$) via panel-specific orthogonal operators ($U_{p(i)}$ and $U_{p(j)}$). This adaptation explicitly guides cross-panel interactions while preserving the original same-panel attention geometry. An example generated image is provided on the right.}
\label{fig:main-ICG}
\end{figure*}

\subsection{Orthogonal Panel-Relative Operator (OPRO)}
\label{sec:method_opro}
While standard PEFT modules or simple linear adapters could be used to inject panel identity, they inherently alter the backbone's feature space. This disruption forces the model to relearn intra-panel attention geometries, degrading the pre-trained generation quality. To explicitly encode panel identity while strictly preserving the native same-panel behavior, we require a transformation that preserves the inner product.

Therefore, we introduce a set of learnable panel operators restricted to the special orthogonal group:
\begin{equation}
\{ U_p\}_{p=1}^P, \qquad U_p \in \mathrm{SO}(d_h).
\end{equation}
By definition, these operators satisfy the orthogonality condition $U_p^\top U_p = I$.

Given the frozen, position-aware query and key representations of the backbone, OPRO applies the panel operator associated with each token:
\begin{equation}
\hat{ q}_i =  U_{p(i)} \tilde{ q}_i,
\qquad
\hat{ k}_j =  U_{p(j)} \tilde{ k}_j.
\label{eq:opro_qk}
\end{equation}
The resulting attention score becomes
\begin{equation}
s'_{ij}
=
\frac{\langle \hat{ q}_i, \hat{ k}_j \rangle}{\sqrt{d_h}}
=
\frac{\tilde{ q}_i^\top \left( U_{p(i)}^\top  U_{p(j)}\right)\tilde{ k}_j}{\sqrt{d_h}}.
\label{eq:opro_attn}
\end{equation}
Equation~\eqref{eq:opro_attn} demonstrates that OPRO preserves the position-aware query and key features of the backbone while introducing an explicit \emph{panel-relative} modulation through the relative operator $ U_{p(i)}^\top  U_{p(j)}$. When two tokens belong to the same panel, this relative operator collapses to the identity. Conversely, when they belong to different panels, OPRO learns a panel-specific modulation of the original attention score.

\begin{proposition}[Isometry]
\label{prop:isometry}
For all tokens $i,j$, the OPRO transformation preserves the norms of the position-aware query and key vectors:
\[
\|\hat{ q}_i\| = \|\tilde{ q}_i\|,
\qquad
\|\hat{ k}_j\| = \|\tilde{ k}_j\|.
\]
\end{proposition}

\begin{proof}
By definition, $\hat{ q}_i =  U_{p(i)} \tilde{ q}_i$ and $\hat{ k}_j =  U_{p(j)} \tilde{ k}_j$. Since each $ U_p$ is orthogonal, $ U_p^\top  U_p =  I$. Therefore,
\[
\|\hat{ q}_i\|^2
=
\tilde{ q}_i^\top  U_{p(i)}^\top  U_{p(i)} \tilde{ q}_i
=
\tilde{ q}_i^\top \tilde{ q}_i
=
\|\tilde{ q}_i\|^2,
\]
and likewise $\|\hat{ k}_j\| = \|\tilde{ k}_j\|$.
\end{proof}

\begin{proposition}[Same-Panel Invariance]
\label{prop:same-panel-invariance}
If two tokens $i$ and $j$ belong to the same panel, i.e., $p(i)=p(j)$, then OPRO preserves their original attention score:
\[
\langle \hat{ q}_i, \hat{ k}_j \rangle
=
\langle \tilde{ q}_i, \tilde{ k}_j \rangle.
\]
Equivalently, for the same-panel token pairs, $s'_{ij} = s_{ij}$.
\end{proposition}

\begin{proof}
If $p(i)=p(j)=p$, then from Eq.~\eqref{eq:opro_attn},
\[
\langle \hat{ q}_i, \hat{ k}_j \rangle
=
\tilde{ q}_i^\top  U_p^\top  U_p \tilde{ k}_j
=
\tilde{ q}_i^\top \tilde{ k}_j
=
\langle \tilde{ q}_i, \tilde{ k}_j \rangle.
\]
Thus, the attention score is preserved for all same-panel token pairs.
\end{proof}

These two propositions together yield the structural behavior we seek. \cref{prop:isometry} ensures that OPRO solely rotates the position-aware query and key features of the backbone without changing their norms, thereby avoiding unintended rescaling of the attention logits. \cref{prop:same-panel-invariance} ensures that intra-panel attention remains identical to that of the pre-trained backbone. As a result, OPRO can concentrate its capacity on modulating cross-panel interactions, where panel identity matters, without perturbing the pre-trained same-panel synthesis behavior of the backbone.

\paragraph{Relation to orthogonal relative positional encodings}
Our formulation above is defined for general position-aware queries and keys, and therefore does not require a specific positional encoding family. However, when the underlying architecture employs an orthogonal relative positional encoding (e.g., RoPE~\citep{su2024roformer}, LieRE~\citep{ostmeierliere}, or ComRoPE~\citep{yu2025comrope}), OPRO admits an additional compositional interpretation around the frozen positional operator. We defer this orthogonal-relative derivation, along with the RoPE-style block-diagonal specialization, to the supplementary material.

\subsection{Parameterization and Zero-Interference Initialization}
\label{sec:method_param}

Section~\ref{sec:method_opro} defines OPRO through abstract orthogonal operators $ U_p \in \mathrm{SO}(d_h)$. We now describe an efficient parameterization for learning these operators and an initialization strategy that preserves the pre-trained backbone at the start of fine-tuning.

\paragraph{Low-rank Lie exponential parameterization}
Directly optimizing a transformation matrix on the constrained manifold of orthogonal matrices ($\mathrm{SO}(d_h)$) demands computationally expensive operations, such as Riemannian gradient descent or repeated matrix projections. To handle this efficiently, we optimize an unconstrained skew-symmetric generator in the corresponding Lie algebra $\mathfrak{so}(d_h)$ and recover the orthogonal operator via the matrix exponential map~\citep{lezcano2019cheap}.

Specifically, we define two learnable matrices $L_p, R_p \in \mathbb{R}^{d_h \times \rho}$ for each panel $p$, where $\rho < d_h$ is the rank, and formulate the generator $A_p$ as
\begin{equation}
A_p = L_p R_p^\top - R_p L_p^\top.
\label{eq:generator}
\end{equation}
The orthogonal operator $U_p$ is then obtained through the matrix exponential
\begin{equation}
U_p = \exp(A_p).
\label{eq:matrix_exp}
\end{equation}
Because $A_p^\top = -A_p$, the matrix exponential naturally guarantees orthogonality ($U_p^\top U_p = I$) by construction. This parameterization allows standard optimizers to operate in an unconstrained Euclidean space while yielding a dense cross-channel orthogonal transform in a parameter-efficient manner.

\paragraph{Zero-interference initialization}
As demonstrated by ControlNet~\citep{zhang2023adding}, zero initialization is a highly effective strategy when applying additive or multiplicative modules to pre-trained models. Because OPRO is designed to adapt a frozen backbone, the adapter must not interfere with the pre-trained representations at the optimization step zero. We achieve this by initializing the low-rank matrices as follows:
\begin{equation}
L_p = 0, \quad R_p \sim \mathcal{N}(0, \sigma^2).
\label{eq:zero_init}
\end{equation}
Under this asymmetric initialization, the generator evaluates to $A_p = 0$, yielding $U_p = \exp(0) = I$. Consequently, OPRO begins as an exact identity mapping, leaving the original attention of the backbone completely unchanged at the start of training. Furthermore, this initialization still admits non-zero gradients for $L_p$, ensuring that optimization commences immediately. Detailed derivations of the gradients are provided in the supplementary material.

%% file: sec/4_experiments.tex
\section{Experiments}
\label{sec:experiments}
We conduct experiments to validate OPRO's effectiveness. 
First, in \cref{subsection:exp-2stage}, we introduce a \textbf{two-stage compositional reasoning} task. This controlled proxy task is designed to analyze OPRO's core properties: its ability to preserve pre-trained knowledge (Same-Panel Invariance) while robustly learning new inter-panel rules.  
Second, in \cref{subsection:exp-main-editing}, we evaluate OPRO on the real-world task of \textbf{instructional image editing} within a two-panel layout. We demonstrate that OPRO, when integrated as a lightweight module, consistently enhances the performance of state-of-the-art diffusion-based editing methods~\citep{zhang2025context, uno}. Finally, to demonstrate OPRO's scalability beyond this two-panel setup, we extend our evaluation to a three-panel layout for \textbf{subject-driven generation} in the supplementary material.

\begin{table*}[t]
\centering
\scriptsize
\caption{Accuracy of LoRA ($r=8$) with/without OPRO ($\rho=8$) across panels. $\Delta$ is the absolute gain over LoRA.}
\label{tab:toy_main_quant} 
\setlength{\tabcolsep}{4.0pt}
\renewcommand{\arraystretch}{1.15}
\begin{minipage}{.32\linewidth}
\centering
\textbf{Panel $2 \times 2$}\\[2pt] 
\begin{tabular}{lcccc}
\toprule
Type & MLP & LoRA & +OPRO & $\Delta$ \\
\midrule
APE~\citep{vaswani2017attention} & 34.90 & \underline{38.00} & \textbf{40.50} & $+2.50$ \\
RoPE~\citep{su2024roformer}       & 37.60 & \underline{46.40} & \textbf{49.90} & $+3.50$ \\
LieRE~\citep{ostmeierliere}       & 33.40 & \textbf{58.10} & \textbf{58.10} & $+0.00$ \\
ComRoPE~\citep{yu2025comrope}   & 38.60 & \underline{58.50} & \textbf{66.60} & $+8.10$ \\
\bottomrule
\end{tabular}
\end{minipage}
\hfill
\begin{minipage}{.32\linewidth}
\centering
\textbf{Panel $3 \times 3$}\\[2pt]
\begin{tabular}{lcccc}
\toprule
Type & MLP & LoRA & +OPRO & $\Delta$ \\
\midrule
APE~\citep{vaswani2017attention}       & 23.70 & \underline{24.40} & \textbf{26.30} & $+1.90$ \\
RoPE~\citep{su2024roformer}       & 27.40 & \underline{36.20} & \textbf{42.00} & $+5.80$ \\
LieRE~\citep{ostmeierliere}       & 29.00 & \underline{34.20} & \textbf{42.70} & $+8.50$ \\
ComRoPE~\citep{yu2025comrope}   & 24.80 & \underline{37.80} & \textbf{45.70} & $+7.90$ \\
\bottomrule
\end{tabular}
\end{minipage}
\hfill
\begin{minipage}{.32\linewidth}
\centering
\textbf{Panel $4 \times 4$}\\[2pt]
\begin{tabular}{lcccc}
\toprule
Type & MLP & LoRA & +OPRO & $\Delta$ \\
\midrule
APE~\citep{vaswani2017attention}       & \underline{20.00} & 19.50 & \textbf{24.20} & $+4.70$ \\
RoPE~\citep{su2024roformer}       & 19.60 & \underline{30.30} & \textbf{39.20} & $+8.90$ \\
LieRE~\citep{ostmeierliere}       & 20.10 & \underline{22.90} & \textbf{26.70} & $+3.80$ \\
ComRoPE~\citep{yu2025comrope}   & 18.90 & \underline{29.20} & \textbf{47.20} & $+18.00$ \\
\bottomrule
\end{tabular}
\end{minipage}
\end{table*}

\subsection{Two-Stage Compositional Reasoning Task}
\label{subsection:exp-2stage}
\begin{figure*}[h!]
  \centering
  \includegraphics[width=0.95\linewidth, trim=0in 1.1in 0in 1.1in, clip]{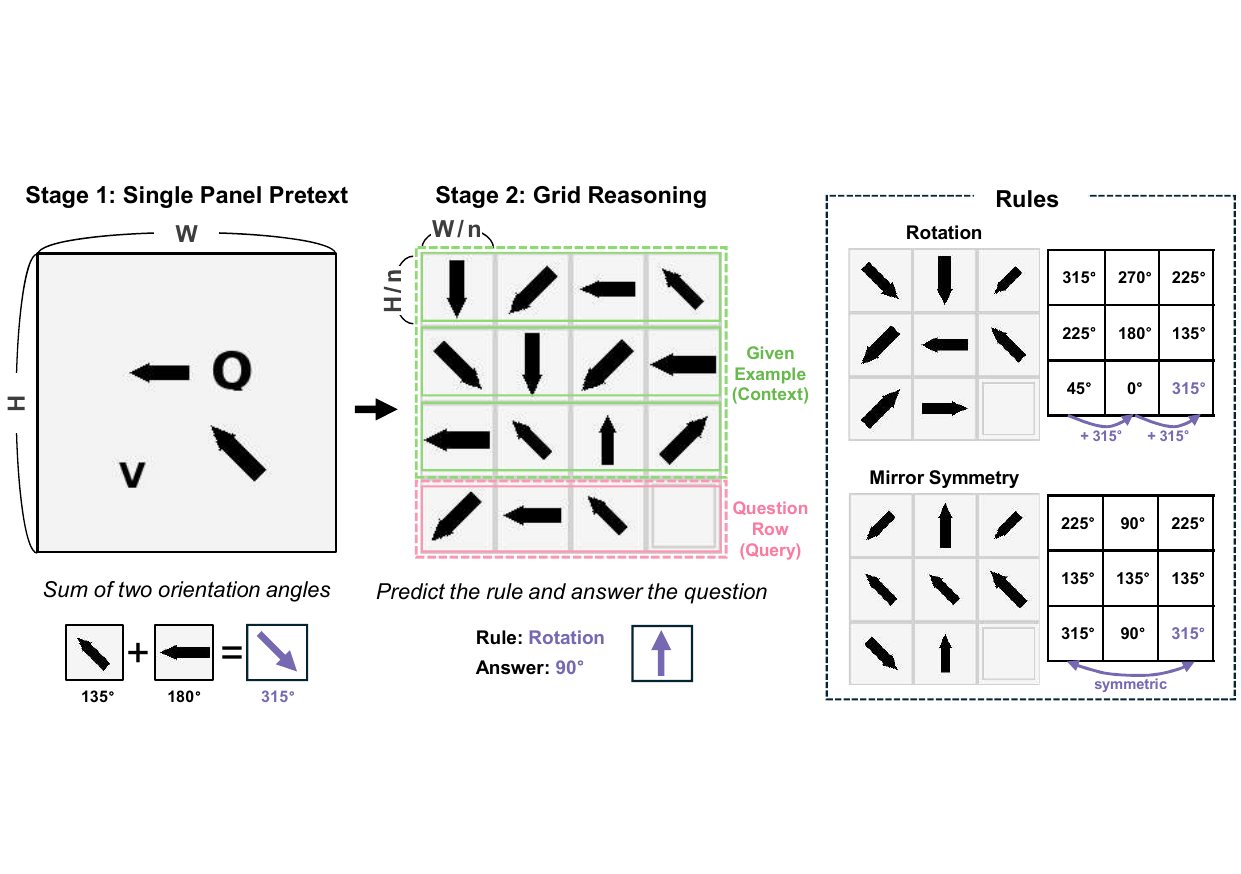}  
  \caption{\textbf{Two-stage compositional reasoning.} Stage 1 (single-panel pretext): classify the sum of two arrow orientations modulo \(360^\circ\) (8-way) with distractors. Stage 2 (grid reasoning): on an \(n{\times}n\) grid, each row provides context examples and a held-out query; the row-wise rule is either rotation by \(k\!\cdot\!45^\circ\) or vertical mirror symmetry.}
  \label{fig:exp-2stage}
\end{figure*}

ICG requires a model to infer a latent, episode-specific rule from visual examples and apply it to a query. This process demands two distinct abilities: (1) robust intra-panel perception to understand the content within each panel, (2) inter-panel reasoning to infer the relationship between panels.

We design a synthetic two-stage benchmark (\cref{fig:exp-2stage}) motivated by Raven-style puzzles~\citep{raven} to concisely evaluate these abilities using the proposed method. In Stage 1 (single-panel pretext), we equip the model with core perceptual skills (e.g., arrow detection and angle composition). In Stage 2 (grid reasoning), we then test a hypothesis: can the model preserve its pre-trained Stage 1 skills when deployed in a multi-panel context while simultaneously learning the new inter-panel reasoning task in its adapters?
Throughout the evaluation, we can measure how effectively different adapters (OPRO vs. LoRA) handle the panel-based structure by replacing pixel generation with classification.

\paragraph{Stage 1: Single-Panel Pretext Task}
Stage 1: Pretrain the model from scratch on a single-panel geometric task. As shown on the left of \cref{fig:exp-2stage}, each image contains a single panel with two arrows placed at random, non-overlapping positions. Arrow orientations are sampled from $\{0^\circ, 45^\circ, \dots, 315^\circ\}$. The label is the sum of the two orientations modulo $360^\circ$, quantized into 8 classes. To prevent shortcutting, we introduce distractors by rendering arrows at random scales and scattering letters across the panel. This pretext induces key intra-panel skills, consisting of robust arrow detection and geometric composition, specifically angle addition.

\paragraph{Stage 2: Grid Reasoning} Stage 2 tests the model's ability to learn inter-panel visual reasoning rules while preserving its pre-trained Stage 1 skills. We partition the image into an $n{\times}n$ grid ($n{\in}\{2,3,4\}$). Crucially, as described in \cref{fig:exp-2stage}, \textbf{each row functions as an independent reasoning problem.}
For a given row $i$, the first $n-1$ panels serve as visual context (examples), and the final $n$-th panel is the held-out query (question). 
The model's objective is to: (1) infer a hidden visual transformation rule from the context panels, (2) apply the found rule to the visual content (the two arrows) in the query.

We consider two categories of visual transformation rules: (i) \textbf{Rotation}: 8 distinct rules defined by visually rotating the panel content (both arrows) by $k\cdot45^\circ$ relative to the previous panel, where $k{\in}\mathbb{N}$. (ii) \textbf{Mirror Symmetry}: 2 distinct rules correspond to a visual reflection of the panel content about the vertical canvas axis (e.g., in a 3-panel row, the third panel is a reflection of the first). To avoid ambiguity that can arise with reflections under 8-way quantization in the limited $2{\times}2$ context, we use only the Rotation rules for $n=2$ panels. For $n \in \{3, 4\}$, we uniformly sample from the complete set of both Rotation and Mirror Symmetry rules.

\paragraph{Experimental Setup}
We design a two-stage experiment to evaluate OPRO's effectiveness in fine-tuning for a multi-panel task. All experiments are based on ViT-B~\citep{vit}.

\textbf{Stage 1 (Backbone pretraining).}
We train four backbones from scratch on a single-panel pretext task, differing only in positional encodings:
APE~\citep{vaswani2017attention}, RoPE~\citep{su2024roformer}, LieRE~\citep{ostmeierliere}, and ComRoPE~\citep{yu2025comrope}.

\textbf{Stage 2 (Multi-panel fine-tuning).}
We initialize from Stage 1 and freeze the backbone, then fine-tune on an 8-way grid-reasoning classification task under three settings:

\begin{itemize} 
\item \textbf{Linear Probe:} Only the classification head (an MLP) is trained. 
\item \textbf{LoRA:} Only LoRA is trained (rank r=8). 
\item \textbf{LoRA + OPRO (Ours):} Both LoRA (r=8) and our OPRO (rank $\rho=8$) 
\end{itemize}
Both stages use \(50\text{k}\) training and \(1\text{k}\) validation images at \(224\times224\).
We use a batch size of 256, Adam~\citep{kingma2014adam}, and cross-entropy loss.
We report top-1 validation accuracy (chance level 12.5\%).

\paragraph{Robustness to Positional Encodings}
\cref{tab:toy_main_quant} summarizes accuracy across positional-encoding backbones and grid sizes. LoRA+OPRO improves over LoRA in the majority of settings, with gains that generally increase with task difficulty (larger $n{\times}n$); improvements reach up to $+18.0\%$ for ComRoPE at $4{\times}4$. Gains are observed with both absolute (APE) and relative (orthogonal) (RoPE, LieRE, ComRoPE) encodings, indicating that the operator is not tied to a specific positional scheme.

\begin{figure}[t]
  \centering
  \includegraphics[width=1\linewidth]{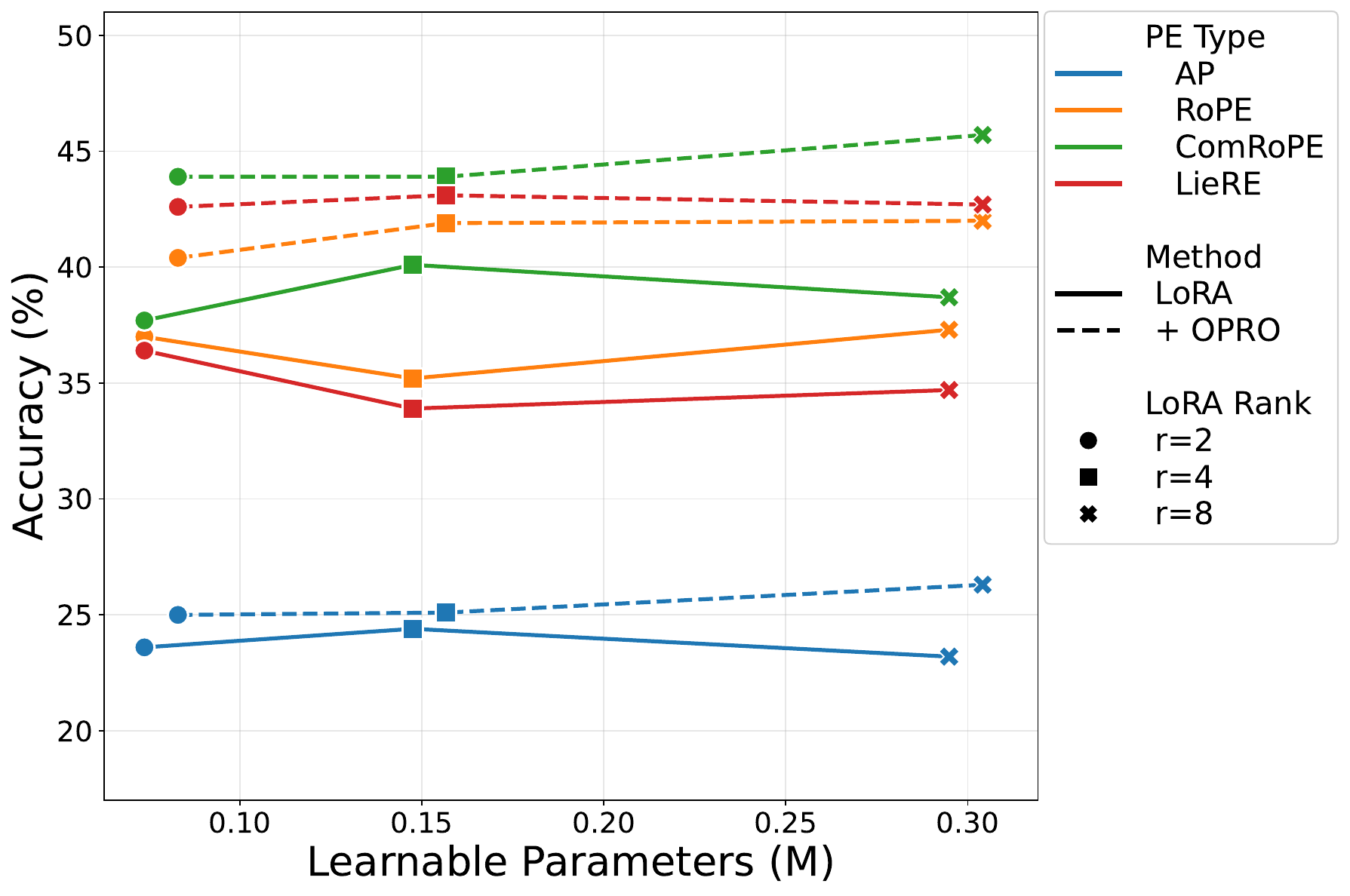}
  \caption{\textbf{OPRO's impact on parameter efficiency.} Validation accuracy (\%) plotted against the number of trainable adapter parameters (M) for 3×3 panels}
  \label{fig:exp-toy-param}
\end{figure}

\paragraph{Parameter Efficiency}
We measure the accuracy–parameter trade-off on $3\times3$ grids (\cref{fig:exp-toy-param}) using a frozen ViT-B backbone (86.6M parameters). LoRA adds trainable parameters that scale linearly with its rank $r$, amounting to 1.327M for $r=8$. Our OPRO, which scales with its own rank $\rho$, introduces a minimal overhead. For instance, at $\rho=8$, OPRO adds only 0.111M parameters. This overhead is negligible, representing just 8.4\% of the LoRA parameters (at $r=8$) and \textbf{0.13\% of the backbone size}. As shown in \cref{fig:exp-toy-param}, this high parameter efficiency allows LoRA+OPRO to achieve a significantly better accuracy–parameter trade-off than LoRA alone.

\paragraph{Ablation Studies}
To validate that the performance gains of OPRO stem from its specific structural properties, we conducted ablation studies on two key design principles: \textbf{(i) Isometry} (norm-preserving transformation) and \textbf{(ii) Same-Panel Invariance} (SP-Inv). We created two variants to isolate the impact of these properties. All variants, including OPRO, were tested at $\rho=4$ for a fair comparison.

\textbf{1) Additive Panel Bias (APB)}. Replace OPRO with per‑panel additive biases for queries/keys. This design is non-isometric (vector addition alters norms) and violates SP-Inv (the biases distort the pre-trained Stage 1 scores):
\[
\langle \tilde{Q}_i + b^Q_p, \tilde{K}_j + b^K_p \rangle 
= \langle \tilde{Q}_i, \tilde{K}_j \rangle 
+ \langle \tilde{Q}_i, b^K_p \rangle 
\]
\[
+ \langle b^Q_p, \tilde{K}_j \rangle 
+ \langle b^Q_p, b^K_p \rangle.
\]

\textbf{2) Asymmetric Orthogonal Operator}: In this variant, we learn independent orthogonal operators for queries and keys, \( U_p, V_p \in \mathrm{SO}(d_h) \):
\[
\hat{Q}_i = U_{p(i)} \tilde{Q}_i, \quad \hat{K}_j = V_{p(j)} \tilde{K}_j.
\]
This design maintains isometry, but breaks SP-Inv unless \( U_p = V_p \), as the inner product becomes:
\[
\langle \hat{Q}_i, \hat{K}_j \rangle = \langle \tilde{Q}_i, (U_p^\top V_p) \tilde{K}_j \rangle.
\]
For SP-Inv across all same-panel pairs, we require \( U_p^\top V_p = I \), which implies \( U_p = V_p \). However, independent learning generally violates this condition.

\begin{table}[t]
\centering
\caption{\textbf{Ablation on 2-stage compositional reasoning task (ViT-B, $3{\times}3$).} We analyze the impact of component removal. \textbf{APB} lacks isometry, and \textbf{Asym-OPRO} violates SP-Inv. \textbf{w/o Zero Init} denotes OPRO with random initialization. OPRO (Ours) satisfies all properties.}\label{tab:toy_ablation}
\resizebox{\columnwidth}{!}{%
\begin{tabular}{l c c c}
\toprule
Method & Isometry & SP-Inv & Accuracy \\
\midrule
LoRA (Baseline) & - & - & 36.20 \\
\midrule
+ APB & No & No & 35.70 \\
+ Asym-OPRO & Yes & No & 39.70 \\
+ OPRO (w/o Zero Init) & Yes & Yes & 38.60 \\
\midrule
\textbf{+ OPRO (Ours)} & \textbf{Yes} & \textbf{Yes} & \textbf{42.00} \\
\bottomrule
\end{tabular}%
}
\end{table}

\begin{figure*}[t]
  \centering
  \includegraphics[width=\textwidth]{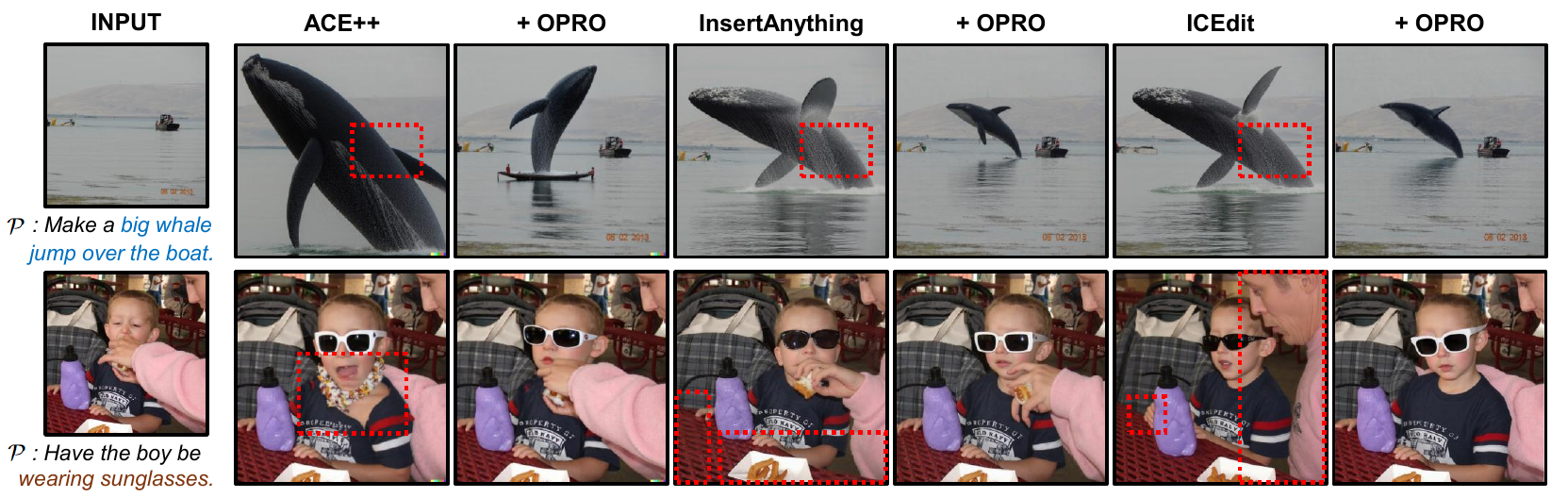}
    \caption{
    \textbf{Comparison with inpainting-based ICG baselines on MagicBrush~\citep{zhang2023magicbrush} test set.}
    Following ICEdit~\citep{zhang2025context}, a diptych prompt is used: \textit{“A diptych with two side-by-side images \ldots but \{ $\mathcal{P}$ \}.”}
    Red dotted boxes highlight regions where the baselines fail to preserve context from the input image, resulting in incorrect or incomplete edits.}  \label{fig:exp-main-opro-qualitative}
\end{figure*}

\begin{figure*}[t]
  \centering
  \includegraphics[width=\linewidth]{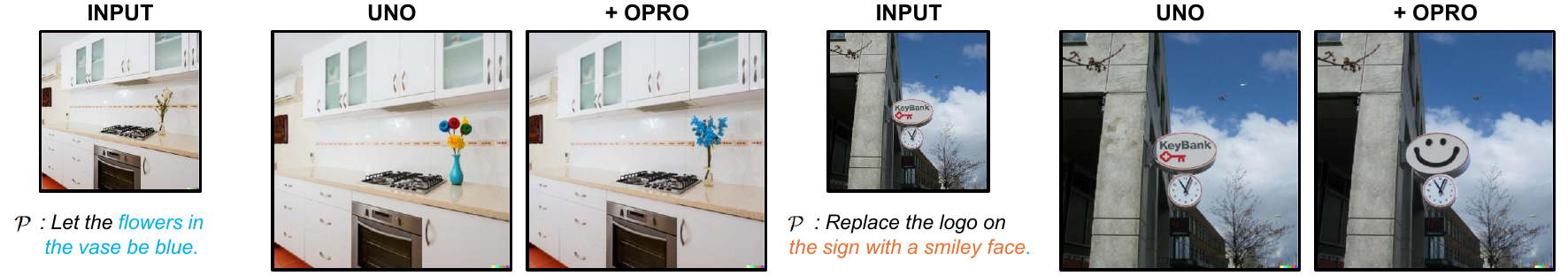}
    \caption{
    \textbf{Comparison with a T2I baseline on MagicBrush~\citep{zhang2023magicbrush} test set.}
    In UNO~\citep{uno}, we use the following instruction: \textit{“Create a single image matching the reference, but with the following edit: \{ $\mathcal{P}$ \}.”}}
    \label{fig:main_qual_uno}
\end{figure*}

\subsection{Instructional Image Editing with OPRO}
\label{subsection:exp-main-editing}
While \cref{subsection:exp-2stage} provides a controlled benchmark, the primary objective is to evaluate OPRO on real-world in-context image generation. Specifically, instructional image editing is formulated as a task comprising two panels: a source reference panel and a target query panel. To assess editing capabilities on this layout, the MagicBrush~\citep{zhang2023magicbrush} test set is utilized. Following prior evaluation protocols~\citep{zhang2023magicbrush,zhang2025context}, visual consistency and editing fidelity are measured using L1, CLIP-I~\citep{clipscore}, and DINO~\citep{oquab2023dinov2,caron2021emerging} metrics.

\paragraph{Experimental Setup}
To demonstrate that OPRO is agnostic to backbone architectures and positional-encoding schemes, we integrate it as a lightweight module (+0.93M parameters, $\rho=32$) into state-of-the-art baselines representing both ICG paradigms. For inpainting-based methods (global-canvas encoding), we evaluate ICEdit~\citep{zhang2025context}, ACE++~\citep{mao2025ace++}, and InsertAnything~\citep{song2025insert}, instantiated on FluxFill~\citep{fluxfill}. For T2I-based methods (per-panel encoding), we evaluate UNO~\citep{uno}, instantiated on FLUX.1. We train all models for 5,000 steps using the Adam optimizer~\citep{kingma2014adam} with a learning rate of $1\times10^{-4}$ and a batch size of 8. The spatial layout and text conditioning are adapted to each paradigm. For inpainting baselines, source images are resized to $512 \times 512$ and placed on the left half of a $512 \times 1024$ canvas, trained with right-half masking, and prompted following ICEdit: \textit{“A diptych with two side-by-side images of the same scene. On the right, the scene is identical to the left but {instruction}.”} For the T2I-based UNO, we follow its native per-panel setup and utilize a direct instruction prompt: \textit{“Create a single image identical to the reference but with the following edit: {instruction}.”}

\begin{table}[t]
\centering
\caption{Quantitative results on MagicBrush~\citep{zhang2023magicbrush}
test set. OPRO ($\rho=32$) adds only +0.93M learnable parameters and consistently improves.}
\label{tab:exp-main-real}
\setlength{\tabcolsep}{3pt}
\resizebox{\columnwidth}{!}{
\begin{tabular}{l c c c c}
\toprule
Methods & Train. Pa & L1 ↓ & CLIP-I ↑ & DINO ↑ \\
\midrule
ACE++ [arXiv'25] & 76.6M & 0.1215 & 0.8658 & 0.7394 \\
\textbf{+ OPRO ($\rho=32$)} & +0.93M & \textbf{0.1114} & \textbf{0.8749} & \textbf{0.7767} \\
\midrule
InsertAnything [arXiv'25] & 37.5M & 0.1327 & 0.8722 & 0.7917 \\
\textbf{+ OPRO ($\rho=32$)} & +0.93M & \textbf{0.1269} & \textbf{0.8735} & \textbf{0.8009} \\
\midrule
ICEdit [NeurIPS'25] & 22.4M & 0.1189 & 0.8703 & 0.7706 \\
\textbf{+ OPRO ($\rho=32$)} & +0.93M & \textbf{0.0781} & \textbf{0.9002} & \textbf{0.8531} \\
\midrule
UNO [ICCV'25] & 478.2M & 0.0575 & 0.9236 & 0.8961 \\
\textbf{+ OPRO ($\rho=32$)} & +0.93M & \textbf{0.0387} & \textbf{0.9281} & \textbf{0.8980} \\

\bottomrule
\end{tabular}%
}
\end{table}

\paragraph{Quantitative Results} As \cref{tab:exp-main-real} demonstrates, adding OPRO consistently improves performance across all baselines while introducing only +0.93M parameters. This overhead corresponds to 4.16\% for ICEdit (22.4M) and 0.20\% for UNO (478.2M), indicating that the additional cost remains negligible across markedly different model scales.

The largest improvement is observed for ICEdit. OPRO reduces L1 by 34.31\% (0.1189 $\rightarrow$ 0.0781), while increasing CLIP-I from 0.8703 to 0.9002 and DINO from 0.7706 to 0.8531. Qualitative results are in \cref{fig:exp-main-opro-qualitative} and \cref{fig:main_qual_uno}.

%% file: sec/5_conclusions.tex
\section{Limitations}
 OPRO introduces a small number of additional orthogonal transforms in the attention layers, which increase training time and inference latency. We view this as a reasonable computational trade-off for improved in-context behavior. Additionally, the model is trained on a fixed panel layout. While we can handle multi-reference configurations during inference by reusing learned operators for panels with shared functional roles (detailed in the supplementary material), handling entirely new layouts that require distinct operators remains an area for future work.

\section{Conclusion}
We propose OPRO, a panel-relative orthogonal adapter for multi-panel in-context image generation. By applying learnable, panel-specific orthogonal operators to the backbone's frozen, position-aware queries and keys, OPRO preserves the feature geometry while cleanly decoupling cross-panel retrieval from intra-panel synthesis. Furthermore, OPRO consistently outperforms standard LoRA in both real-world instructional image editing and our proposed compositional reasoning benchmark.

\section*{Acknowledgments}
This work was supported by Institute for Information \& communications Technology Planning \& Evaluation(IITP) grant funded by the Korea government(MSIT) (RS-2019-II190075, Artificial Intelligence Graduate School Program(KAIST)), the National Research Foundation of Korea(NRF) grant funded by the Korea government(MSIT) (No. RS-2025-00555621), and i-Scream Media. This research was also supported by the High-Performance Computing Support Project, funded by the Ministry of Science and ICT (MSIT) and the National IT Industry Promotion Agency (NIPA) under grant No. RQT-25-070278 (providing 40 H100 GPUs).

%% file: sec/X_suppl_arxiv.tex
\appendix

\twocolumn[{%
\renewcommand\twocolumn[1][]{#1}%
\maketitlesupplementary
\begin{center}
    \centering
    \includegraphics[width=\linewidth]{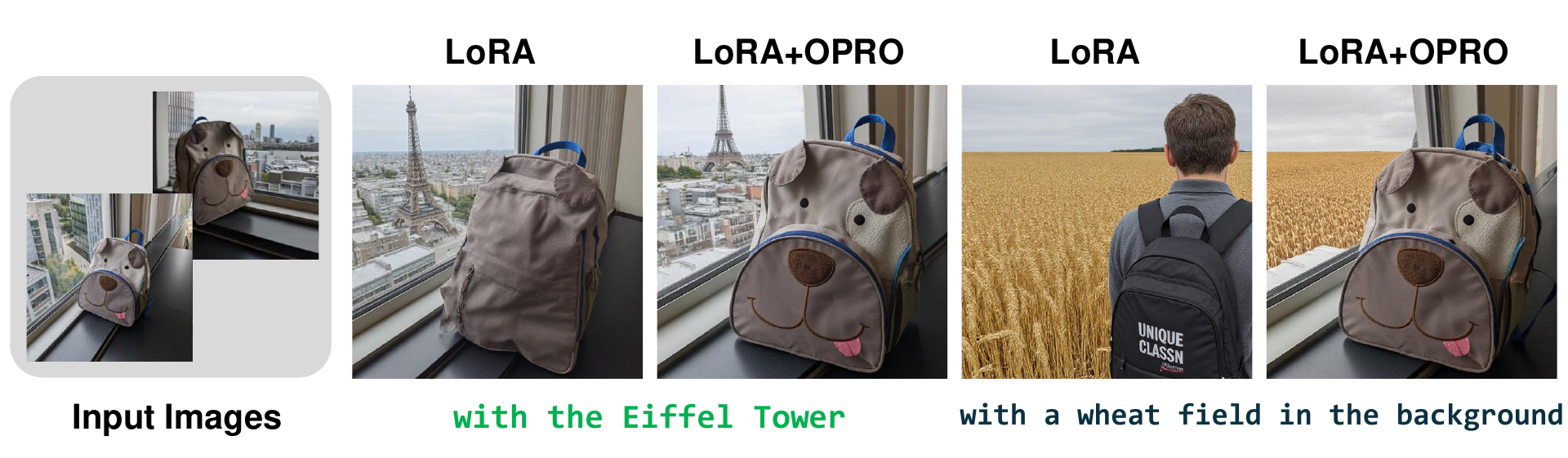}
    \captionsetup{hypcap=false}
    \captionof{figure}{
    \textbf{Qualitative comparison on subject-driven image generation.}
    Results are shown on the DreamBooth~\citep{ruiz2023dreambooth} test set under a three-panel protocol. For each subject, two reference images sampled from a four-shot support set occupy the first two panels, and the third panel is synthesized from a fully masked target canvas. We compare LoRA-only fine-tuning of ICEdit~\citep{zhang2025context} with the same model modulated by OPRO.
  }
\label{fig:sup-qual-subject}
\end{center}%
}]

This supplementary material provides additional empirical results, analyses, and implementation details that complement the main manuscript. Section~\ref{sec:sup-subject} evaluates OPRO on subject-driven image generation in a three-panel setting. Section~\ref{sec:sup-derivation} presents a RoPE-aligned block-diagonal parameterization and its formal derivation, and Section~\ref{sec:sup-init} provides a detailed analysis of the zero-initialization strategy. Section~\ref{sec:sup-cost} analyzes the computational overhead of OPRO. Section~\ref{sec:sup-qualitative} presents additional qualitative results and examines inference-time scalability. Section~\ref{sec:sup-real-ablation} reports ablation studies on instructional image editing, and Section~\ref{sec:sup-hyper} summarizes the complete hyperparameter settings used in the experiments.

\section{Subject-Driven Image Generation with OPRO}
\label{sec:sup-subject}
To assess the scalability of OPRO beyond the two-panel setting in the main manuscript, we evaluate subject-driven image generation in a \textbf{three-panel layout} leveraging the DreamBooth~\citep{ruiz2023dreambooth} test dataset. For each subject, we construct a four-shot support set and randomly sample two reference images to populate the first two panels. The third panel serves as a fully masked target canvas. We adopt ICEdit~\citep{zhang2025context} as the base model and integrate OPRO as a lightweight panel-relative adaptation module. This task places a stronger emphasis on cross-panel subject consistency because the target panel must be synthesized from scratch while aggregating subject cues from multiple reference panels. Optimization proceeds for 2,000 steps with Adam~\citep{kingma2014adam} at a learning rate of $1\times10^{-4}$.   

Table~\ref{tab:sup-subject-quant} shows that OPRO improves ICEdit on both DINO and CLIP-I, with absolute gains of 0.0364 and 0.0348, respectively. Figure~\ref{fig:sup-qual-subject} further illustrates more faithful preservation of subject appearance and more coherent synthesis of the target panel than the LoRA-only baseline.

\begin{table}[t]
\centering
\caption{
    \textbf{Quantitative comparison on subject-driven image generation.}
    Results are reported on a subset of DreamBooth using a three-panel layout with two reference panels and one fully masked target panel. OPRO consistently improves ICEdit on both DINO and CLIP-I. Higher is better in all cases.
    }
    \label{tab:sup-subject-quant}
\begin{tabular}{lcc}
\toprule
\textbf{Method} & \textbf{DINO} ($\uparrow$) & \textbf{CLIP-I} ($\uparrow$) \\
\midrule
ICEdit~\citep{zhang2025context} (LoRA-only) & 0.5828 & 0.7376 \\
ICEdit~\citep{zhang2025context} + OPRO & \textbf{0.6192} & \textbf{0.7724} \\
\bottomrule
\end{tabular}
\end{table}

\section{RoPE-Aligned Block-Diagonal Parameterization}
\label{sec:sup-derivation}
This section complements~\cref{sec:method_opro} of the main manuscript by detailing the relationship between OPRO and orthogonal relative positional encodings~\citep{su2024roformer, ostmeierliere, yu2025comrope}. As briefly discussed in the main text, OPRO admits an additional compositional interpretation around the frozen positional operator. We first derive this general orthogonal-relative form and then present a RoPE-aligned block-diagonal specialization, which yields the panel-relative phase-shift interpretation. This specialization is introduced for analysis and intuition; the trainable parameterization used in the main experiments is the low-rank Lie exponential parameterization in~\cref{sec:method_param} of the main manuscript.

\paragraph{General orthogonal-relative form.}
For completeness, the frozen position-aware vectors are expressed in matrix form.
Let \(q_i, k_j \in \mathbb{R}^{d_h}\) denote the content vectors before the frozen positional transform, and let the backbone positional mechanism be represented by an orthogonal operator \(R(\mathbf{x}) \in \mathrm{SO}(d_h)\):
\[
\tilde q_i = R(\mathbf{x}_i) q_i,
\qquad
\tilde k_j = R(\mathbf{x}_j) k_j .
\]
Assume that \(R(\mathbf{x})\) satisfies the relative-position property
\[
R(\mathbf{x}_i)^\top R(\mathbf{x}_j) = R(\mathbf{x}_j - \mathbf{x}_i).
\]
Applying OPRO gives
\[
\hat q_i = U_{p(i)} \tilde q_i,
\qquad
\hat k_j = U_{p(j)} \tilde k_j,
\]
and therefore
\[
\langle \hat q_i, \hat k_j \rangle
=
q_i^\top
R(\mathbf{x}_i)^\top
U_{p(i)}^\top U_{p(j)}
R(\mathbf{x}_j)
k_j.
\]
This expression shows that OPRO preserves the frozen positional operator while inserting a learnable panel-relative orthogonal factor \(U_{p(i)}^\top U_{p(j)}\).

\paragraph{RoPE-aligned block-diagonal specialization.}
To obtain a closed-form phase interpretation, we consider a stronger specialization in which \(U_p\) is restricted to the same block-diagonal \(\mathrm{SO}(2)\) basis as RoPE. Let \(d_h\) be even and write
\[
R(\mathbf{x})
=
\mathrm{diag}\!\Big(
R^{(1)}(\theta_1(\mathbf{x})),\,
\dots,\,
R^{(d_h/2)}(\theta_{d_h/2}(\mathbf{x}))
\Big),
\]
where each \(R^{(k)}(\theta)\in\mathrm{SO}(2)\) is a \(2\times2\) rotation.
We parameterize
\[
U_p
=
\mathrm{diag}\!\Big(
R^{(1)}(\phi_{p,1}),\,
\dots,\,
R^{(d_h/2)}(\phi_{p,d_h/2})
\Big).
\]
Because \(R(\mathbf{x})\) and \(U_p\) are block-diagonal rotations acting on the same two-dimensional channel pairs, they commute:
\[
U_p R(\mathbf{x}) = R(\mathbf{x}) U_p.
\]
Hence,
\[
\begin{aligned}
\langle \hat q_i, \hat k_j \rangle
&=
q_i^\top
R(\mathbf{x}_i)^\top
U_{p(i)}^\top U_{p(j)}
R(\mathbf{x}_j)
k_j \\
&=
q_i^\top
R(\mathbf{x}_i)^\top
R(\mathbf{x}_j)
U_{p(i)}^\top U_{p(j)}
k_j \\
&=
q_i^\top
R(\mathbf{x}_j-\mathbf{x}_i)
U_{p(i)}^\top U_{p(j)}
k_j .
\end{aligned}
\]
Moreover,
\[
\begin{split}
U_{p(i)}^\top U_{p(j)} = \mathrm{diag}\!\Big( & R^{(1)}(\phi_{p(j),1}-\phi_{p(i),1}),\, \dots, \\
& R^{(d_h/2)}(\phi_{p(j),d_h/2}-\phi_{p(i),d_h/2}) \Big),
\end{split}
\]
so the effective angle of the \(k\)-th block is
\[
\theta_k(\mathbf{x}_j-\mathbf{x}_i)
+
\phi_{p(j),k}-\phi_{p(i),k}.
\]
Therefore, in this RoPE-aligned block-diagonal specialization, OPRO injects a learnable panel-relative phase offset into each frequency block.

\begin{table*}[t]
\centering
\scriptsize
\small
\caption{Effect of the block-diagonal implementation of OPRO on top of LoRA ($r=8$). We report the accuracy (\%) of LoRA+OPRO-BD and the absolute change $\Delta$ (percentage points) compared to the LoRA baseline from Tab.~1 of the main manuscript.}
\label{tab:supp_plain_opro}
\setlength{\tabcolsep}{4.5pt}
\renewcommand{\arraystretch}{1.15}

\begin{tabular}{lcccccc}
\toprule
& \multicolumn{2}{c}{Panel $2 \times 2$}
& \multicolumn{2}{c}{Panel $3 \times 3$}
& \multicolumn{2}{c}{Panel $4 \times 4$} \\
\cmidrule(lr){2-3}\cmidrule(lr){4-5}\cmidrule(lr){6-7}
\textbf{Type} 
& \textbf{+OPRO-BD} & $\boldsymbol{\Delta}$ 
& \textbf{+OPRO-BD} & $\boldsymbol{\Delta}$ 
& \textbf{+OPRO-BD} & $\boldsymbol{\Delta}$ \\
\midrule
APE      
& 37.10 & \negt{0.90}
& 23.60 & \negt{0.80}
& 19.00 & \negt{0.50} \\
RoPE\citep{su2024roformer}     
& 45.80 & \negt{0.60}
& 38.70 & \post{2.50}
& 32.50 & \post{2.20} \\
LieRE\citep{ostmeierliere}
& 58.70 & \post{0.60}
& 36.20 & \post{2.00}
& 23.30 & \post{0.40} \\
ComRoPE\citep{yu2025comrope}
& 57.90 & \negt{0.60}
& 40.90 & \post{3.10}
& 29.80 & \post{0.60} \\
\bottomrule
\end{tabular}
\end{table*}

\paragraph{Validation on Compositional Reasoning Task}
Table~\ref{tab:supp_plain_opro} summarizes the performance of the block-diagonal parameterization implementation (OPRO-BD) applied to the two-stage compositional reasoning task. With only a minimal parameter overhead equal to the number of panels ($P=4, 9, 16$), OPRO-BD demonstrates improvements for orthogonal positional encodings in the $3\times3$ and $4\times4$ panel experiments.

\section{Detailed Analysis of Zero Initialization Strategy}
\label{sec:sup-init}
In this section, we provide a detailed analysis of the zero-initialization strategy.
We first formally prove that our parameterization guarantees non-degenerate gradients. 

 Recall from~\cref{sec:method_param} in the manuscripts that for each panel $p$ we
parameterize the orthogonal operator as
\[
    U_p = \exp(A_p), \qquad A_p = L_p R_p^\top - R_p L_p^\top,
\]
where $L_p, R_p \in \mathbb{R}^{d_h \times r}$ are learnable parameters and
$\exp(\cdot)$ denotes the matrix exponential. We initialize
\[
    L_p = \mathbf{0}, \qquad R_p \sim \mathcal{N}(0, \sigma^2),
\]
so that $A_p = \mathbf{0}$ and $U_p = I$ at step~0. Thus the OPRO operator has
no effect on the pre-trained model at initialization, while still admitting a
non-degenerate gradient, as we show below.

\paragraph{Notation} Let $\mathcal{L}$ be a scalar loss and define the Frobenius inner product
$\langle X, Y \rangle = \mathrm{tr}(X^\top Y)$.
Write
\[
    G := \nabla_{U_p} \mathcal{L}
    \quad\text{and}\quad
    \tilde G := \mathrm{D}\exp_{A_p}^{*}[G],
\]
where $\mathrm{D}\exp_{A_p}$ is the differential of the matrix exponential
at $A_p$ and $\mathrm{D}\exp_{A_p}^{*}$ is its adjoint with respect to
the Frobenius inner product \citep{AbsMahSep2008}.

\begin{proposition}[Zero initialization identity mapping with non-degenerate gradient]
    Let $U_p = \exp(A_p)$ with
    \[
        A_p = L_p R_p^\top - R_p L_p^\top.
    \]
    Then the gradients of $\mathcal{L}$ with respect to $L_p$ and $R_p$ are
    \[
        \nabla_{L_p} \mathcal{L} = (\tilde G - \tilde G^\top)\, R_p,
        \qquad
        \nabla_{R_p} \mathcal{L} = (\tilde G^\top - \tilde G)\, L_p.
    \]
    In particular, at zero initialization ($A_p = \mathbf{0}$ and $L_p = \mathbf{0}$),
    we have $U_p = I$ and $\tilde G = G$, so
    \[
        \nabla_{L_p} \mathcal{L} = (G - G^\top)\, R_p,
        \qquad
        \nabla_{R_p} \mathcal{L} = \mathbf{0}.
    \]
    Thus, the operator is initially the identity, but optimization starts
    immediately through $L_p$, while $R_p$ remains fixed at the first step.
\end{proposition}

\medskip
\noindent\textit{Proof.}
By the chain rule and the expression for the differential of the
matrix exponential \citep{AbsMahSep2008}, for any perturbation
$E$ we have
\[
    d\mathcal{L} = \langle G, \mathrm{D}\exp_{A_p}[dA_p] \rangle = \langle \tilde G, dA_p \rangle.
\]
Differentiating $A_p = L_p R_p^\top - R_p L_p^\top$ gives
\[
    dA_p = dL_p\, R_p^\top + L_p\, dR_p^\top - dR_p\, L_p^\top - R_p\, dL_p^\top.
\]
Substituting this into the inner product and applying the identity $\langle X, Y Z^\top \rangle = \langle X Z, Y \rangle$, we expand $\langle \tilde G, dA_p \rangle$:
\[
\begin{aligned}
    \langle \tilde G, dA_p \rangle
    &= \big\langle \tilde G,\, dL_p\, R_p^\top + L_p\, dR_p^\top - dR_p\, L_p^\top - R_p\, dL_p^\top \big\rangle \\
    &= \big\langle (\tilde G - \tilde G^\top)\, R_p,\, dL_p \big\rangle
     + \big\langle (\tilde G^\top - \tilde G)\, L_p,\, dR_p \big\rangle.
\end{aligned}
\]
By the definition of the gradient with respect to the Frobenius inner product, this implies
\[
    \nabla_{L_p} \mathcal{L} = (\tilde G - \tilde G^\top)\, R_p,
    \qquad
    \nabla_{R_p} \mathcal{L} = (\tilde G^\top - \tilde G)\, L_p.
\]
Under zero initialization $A_p = \mathbf{0}$, the Jacobian of the exponential map is the identity, therefore $\tilde G = G$. With $L_p = \mathbf{0}$, the gradients simplify to:
\[
    \nabla_{L_p} \mathcal{L} = (G - G^\top)\, R_p,
    \qquad
    \nabla_{R_p} \mathcal{L} = \mathbf{0},
\]\hfill$\square$

\begin{figure*}[t]
  \centering
  \includegraphics[width=\linewidth]{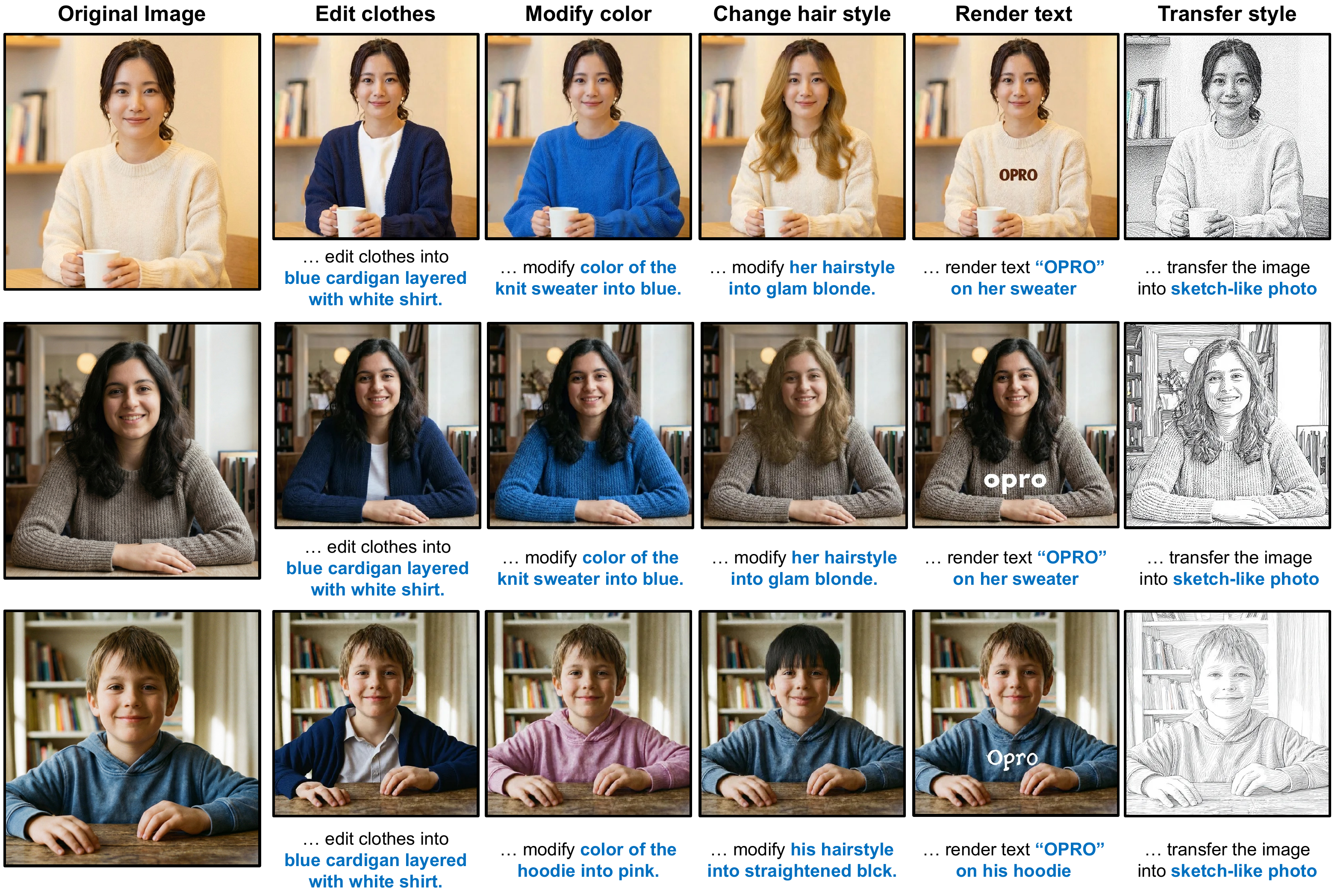}
    \caption{\textbf{Qualitative results on diverse instructional editing tasks.} We demonstrate the versatility of OPRO across a broad spectrum of editing categories. The examples illustrate the model's capability to precisely follow instructions for object replacement, attribute modification, text rendering, and global style transfer, all while maintaining high fidelity to the original image content.}
    \label{fig:sup_qual_main}
\end{figure*}

\begin{figure*}[t]
  \centering
  \includegraphics[width=\linewidth]{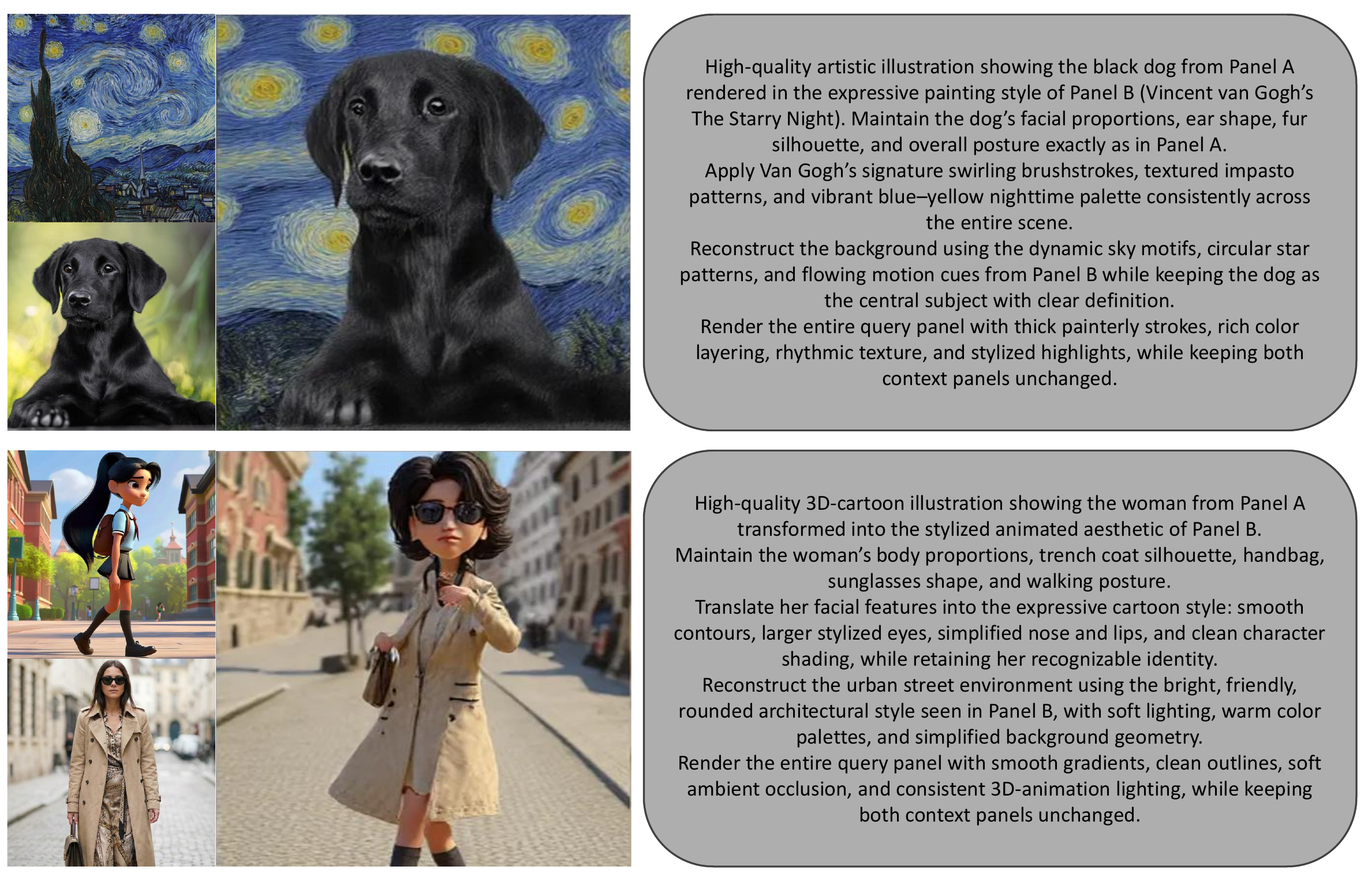}
    \caption{\textbf{Qualitative examples of multi-reference compositional generation.} We demonstrate the capability to integrate attributes from multiple context panels. The model synthesizes a new image by combining the \textit{style} from the first panel and the \textit{object} from the second panel. Note that this compositional ability emerges without explicit training on multi-reference layouts.}
    \label{fig:sup_qual_multi}
\end{figure*}

\section{Computational Cost Analysis}
\label{sec:sup-cost}

We analyze the computational overhead of OPRO when integrated with FluxFill~\citep{fluxfill}. OPRO introduces additional orthogonal transformations within the attention layers. Specifically, at each step and layer, OPRO performs two $128 \times 128$ matrix-vector rotations, corresponding to queries and keys, across all tokens. The additional floating-point operations ($\Delta \text{FLOPs}$) can be approximated as
\begin{equation}
    \Delta \text{FLOPs} \approx N_{\text{panel}} \cdot N_{\text{head}} \cdot d_h^2 \cdot N_{\text{tokens}} \cdot N_{\text{layers}} \cdot N_{\text{steps}},
\end{equation}
where $N_{\text{tokens}}$ denotes the number of tokens per panel. Substituting the configuration parameters detailed in~\cref{subsection:exp-main-editing} of the main manuscript ($N_{\text{panel}}=2$, $N_{\text{head}}=24$, $d_h=128$, $N_{\text{tokens}}=4{,}096$, $N_{\text{layers}}=57$, $N_{\text{steps}}=28$), the total additional computation amounts to approximately 10.3~TFLOPs. Furthermore, the cost of computing the matrix exponential is negligible (approximately 6.7~GFLOPs). Given the substantial computational budget of diffusion transformers, this theoretical overhead remains marginal.

\section{Qualitative Results and Inference-Time Scalability}
\label{sec:sup-qualitative}

We present additional qualitative results generated by ICEdit~\citep{zhang2025context} equipped with OPRO. Figure~\ref{fig:sup_qual_main} demonstrates the versatility of OPRO across multiple instructional editing tasks. The provided examples illustrate the capability of the model to execute precise modifications, including object replacement, attribute alteration, text rendering, and global style transfer, while preserving the content of the original image.

Furthermore,~\cref{fig:sup_qual_multi} details the inference-time scalability of OPRO by demonstrating compositional generation with multi-reference inputs. Specifically, we apply a model trained on a fixed two-panel layout to a three-panel configuration comprising two reference images. We enable this multi-reference inference by reusing the OPRO learned for the single-reference panel across both references, while applying the target operator to the generation panel. By assigning images to panels that share identical functional roles during inference, we achieve multi-reference compositional generation without requiring retraining.

\section{Ablation Studies on Instructional Image Editing}
\label{sec:sup-real-ablation}
\begin{table}[t]
\centering
\caption{
    \textbf{Ablation Studies on MagicBrush.}
    The table validates the design principles of OPRO on the ICEdit baseline ($r{=}16$). Relative to OPRO, breaking isometry (APB) or same-panel invariance (Asym-OPRO) reduces performance. Removing zero initialization preserves spatial alignment but lowers semantic consistency, as reflected by CLIP-I and DINO.
    }
    \label{tab:realworld_ablation}
\resizebox{\columnwidth}{!}{
\begin{tabular}{l c c c c c}
\toprule
Method & Isometry & SP-Inv & L1 $\downarrow$ & CLIP-I $\uparrow$ & DINO $\uparrow$ \\
\midrule
LoRA (Baseline) & - & - & 0.1189 & 0.8703 & 0.7706 \\
\midrule
+ APB & No & No & 0.0966 & 0.8893 & 0.8196 \\
+ Asym-OPRO & Yes & No & 0.0988 & 0.8880 & 0.8151 \\
+ OPRO (w/o Zero Init) & Yes & Yes & \textbf{0.0780} & 0.8989 & 0.8510 \\
\midrule
\textbf{+ OPRO (Ours)} & \textbf{Yes} & \textbf{Yes} & 0.0781 & \textbf{0.9002} & \textbf{0.8531} \\
\bottomrule
\end{tabular}%
}
\end{table}
Table~\ref{tab:realworld_ablation} extends the ablation studies of the main manuscript to the MagicBrush dataset. Consistent with the results of the compositional reasoning task, violating isometry (APB) or same-panel invariance (Asym-OPRO) degrades performance across all metrics. Furthermore, omitting zero initialization (+ OPRO w/o Zero Init) achieves a spatial alignment error (L1) comparable to that of the proposed OPRO, yet yields lower semantic consistency scores (CLIP-I and DINO) than the proposed method. This semantic degradation aligns with the accuracy drop observed in the main manuscript, demonstrating that an identity mapping initialization is essential to preserve the visual priors of the pre-trained model.

\section{Complete Hyperparameter Settings}
\label{sec:sup-hyper}
We provide detailed hyperparameter configurations used in our experiments. \cref{tab:sup-hyperparams_editing} summarizes the training settings for the instructional image editing baselines. \cref{tab:sup-hyperparams_two_stage} presents the optimization details for the two-stage compositional reasoning task.
\begin{table*}[t]
    \centering
    \caption{\textbf{Hyperparameters for instructional image editing baselines.} All models are trained for 5,000 steps using the AdamW optimizer (weight decay 0.01, learning rate $1 \times 10^{-4}$) with a batch size of 8. We use bfloat16 precision and a constant learning rate schedule. Note that InsertAnything uses FluxPriorRedux for reference-image conditioning, while UNO adopts an in-context approach. We adapt OPRO in each self-attention layer.}
    \label{tab:sup-hyperparams_editing}
    \resizebox{\linewidth}{!}{%
    \begin{tabular}{l|l l l c c}
    \toprule
    \textbf{Method} & \textbf{Base Model} & \textbf{Pos. Encoding} & \textbf{LoRA Target Modules} & \textbf{LoRA Rank ($r$)} & \textbf{OPRO Rank ($\rho$)} \\
    \midrule
    ICEdit~\citep{zhang2025context} & FluxFill Dev & Global-canvas & Attention (\texttt{q,k,v,out}) & 16 & 32 \\
    ACE++~\citep{mao2025ace++} & FluxFill Dev & Global-canvas & Attn + MLP + Modulation & 16 & 32 \\
    InsertAnything~\citep{song2025insert} & FluxFill Dev (+Redux) & Global-canvas & Attention Projections (\texttt{q,k,v,out}) & 16 & 32 \\
    UNO~\citep{uno} & Flux Dev & Per-panel & Attention Projections  + MLP & 256 & 32 \\
    \bottomrule
    \end{tabular}
    }
\end{table*}

\begin{table*}[t]
\centering
\small 
\caption{\textbf{Detailed hyperparameters for two-stage compositional reasoning.}}
\label{tab:sup-hyperparams_two_stage}
\begin{tabular}{l|cc}
\toprule
\textbf{Hyperparameter} & \textbf{Stage 1 (Pre-training)} & \textbf{Stage 2 (Fine-tuning)} \\
\midrule
Optimization & Adam ($\beta_1=0.9, \beta_2=0.999$) & Adam ($\beta_1=0.9, \beta_2=0.999$) \\
Batch Size & 256 & 256 \\
Learning Rate & $1 \times 10^{-3}$ (Warmup+Cosine) & $5 \times 10^{-4}$ (Constant) \\
Weight Decay & $0.05$ & $0.05$ \\
Training Steps & 50k & 2k \\
\midrule
\textbf{Architecture / Adapter} & & \\
Patch Size & $16 \times 16$ & $16 \times 16$ \\
Positional Encoding & Learnable (APE/RoPE etc.) & Frozen \\
Adapter Rank & - & LoRA $r=8$ / OPRO $\rho=(2,4,8)$ \\
\bottomrule
\end{tabular}
\end{table*}